\documentclass{article}
\usepackage{hyperref}
\usepackage{amsmath,amssymb,amsfonts,mathrsfs, amsthm}
\usepackage{mathtools}
\usepackage{amsthm}
\usepackage{bm}
\usepackage{fullpage}
\usepackage{graphicx}
\usepackage{subcaption}
\usepackage{natbib}
\usepackage{url, bbm}
\usepackage{cleveref}
\usepackage{enumitem}
\usepackage{color}
\usepackage{tikz}
\usepackage{tkz-graph}
\usepackage{authblk}
\usetikzlibrary{shapes.geometric}
\usetikzlibrary{backgrounds}
\usetikzlibrary{arrows.meta}
\usetikzlibrary{positioning}
\usepackage[framemethod=TikZ]{mdframed}

\usepackage{algorithm,algpseudocode}

\newtheorem{theorem}{Theorem}[section]
\newtheorem{proposition}{Proposition}[section]
\newtheorem{lemma}{Lemma}[section]
\newtheorem{corollary}{Corollary}[section]

\newtheorem{definition}{Definition}[section]

\newtheorem{assumption}{Assumption}

\title{Identifying Latent Actions and Dynamics from Offline Data via Demonstrator Diversity}
\author{Felix Schur\footnote{felix.schur@stat.math.ethz.ch}\\
Department of Mathematics, ETH Zürich}

\begin{document}
\newcommand\norm[1]{\left\lVert#1\right\rVert}
\newcommand{\indep}{\perp \!\!\! \perp}
\newcommand{\rank}{\mathrm{rank}}

\maketitle

\begin{abstract}
Can latent actions and environment dynamics be recovered from offline trajectories when actions are never observed? We study this question in a setting where trajectories are action-free but tagged with demonstrator identity. We assume that each demonstrator follows a distinct policy, while the environment dynamics are shared across demonstrators and identity affects the next observation only through the chosen action. Under these assumptions, the conditional next-observation distribution $p(o_{t+1}\mid o_t,e)$ is a mixture of latent action-conditioned transition kernels with demonstrator-specific mixing weights. We show that this induces, for each state, a column-stochastic nonnegative matrix factorization of the observable conditional distribution. Using sufficiently scattered policy diversity and rank conditions, we prove that the latent transitions and demonstrator policies are identifiable up to permutation of the latent action labels. We extend the result to continuous observation spaces via a Gram-determinant minimum-volume criterion, and show that continuity of the transition map over a connected state space upgrades local permutation ambiguities to a single global permutation. A small amount of labeled action data then suffices to fix this final ambiguity. These results establish demonstrator diversity as a principled source of identifiability for learning latent actions and dynamics from offline RL data.
\end{abstract}

\section{Introduction}

Recent progress in machine learning has been driven in large part by the availability of massive internet-scale datasets. Large language models can be trained on web text because both the input and the prediction target are directly observed. Reinforcement learning, by contrast, has not benefited from web-scale pretraining in the same way. Although the internet contains enormous amounts of sequential data---gameplay videos, robot videos, screen recordings, and human demonstrations---these data typically do not come with the action labels needed to train a policy or a dynamics model in the standard way \cite{seo2022reinforcement,baker2022vpt}. This lack of action annotations is a central obstacle to scaling RL from passive data.

A more fundamental difficulty is that, from observation-only trajectories alone, it is generally impossible to tell whether a change from $o_t$ to $o_{t+1}$ was caused by an action or whether it is simply a random event in the environment. As a simple example, imagine a video of a television screen whose channel changes over time. From the video alone, we cannot tell whether the TV is broken and changes channels by itself, or whether an unseen agent is using a remote control. Both mechanisms can produce the same observed transitions. More generally, if we only observe the marginal conditional distribution $p(o_{t+1}\mid o_t)$ under a single behavior policy, then action choice and environment stochasticity are confounded: many different latent action spaces, policies, and transition kernels can induce exactly the same observed next-observation distribution. Thus, latent actions and action-conditioned dynamics are not identifiable from a single stream of observation-only data without additional structure.
We propose to resolve this ambiguity by exploiting demonstrator diversity. The intuition is again simple. If a channel switch happens with the same probability for all viewers, then it is natural to attribute it to the TV itself rather than to the viewer. But if the transition probabilities differ systematically across viewers, then this variation must be mediated by action choice, provided that viewer identity has no direct effect on the TV beyond the chosen action. In other words, demonstrator identity induces variation in the latent policy while the transition dynamics remain shared. If the demonstrator policies are sufficiently diverse, this variation is rich enough to make the latent actions and dynamics identifiable.
We formalize this idea in a setting where the action $a_t$ is unobserved, but the demonstrator identity $e$ is observed together with $(o_t,o_{t+1})$. We assume that each demonstrator follows a potentially different policy $\pi_e(a\mid o)$, while the environment dynamics $p(o'\mid o,a)$ are shared across demonstrators and identity affects the next observation only through the action. Under these assumptions, for each fixed $o$ the observable conditional distribution $p(\cdot\mid o,e)$ admits a mixture decomposition over latent actions,
\[
p(\cdot \mid o,e) = \sum_{a=1}^k p(\cdot \mid o,a)\,\pi_e(a\mid o).
\]
This yields a column-stochastic factorization of the observable conditionals into latent action-conditioned transition kernels and demonstrator-specific action probabilities.

Our main result shows that when the demonstrator policies are sufficiently diverse, this factorization becomes identifiable up to a permutation of the latent action labels. In the finite case, we cast the problem as a column-stochastic nonnegative matrix factorization and leverage minimum-volume identifiability results under a sufficiently scattered condition \citep{fu2018identifiability,huang2018hmm}. In the continuous case, we derive an analogous result based on minimizing the Gram determinant of the latent transition functions. We further show that if the transition map varies continuously over a connected observation space, then the state-wise permutation ambiguity cannot vary across observations and must therefore be global. Finally, a small amount of labeled action data is enough to fix this remaining symmetry completely.
More broadly, our results suggest a route toward scaling RL from passive sequential data: rather than treating heterogeneous demonstrators as nuisance variation, we can use their differences as a source of information. Demonstrator diversity makes it possible to distinguish what is truly random in the world from what is driven by latent decisions, and thereby to recover action semantics and dynamics from action-free offline data.

\section{Related Work}

A large body of work studies how to learn from demonstrations that do not contain action labels. Early \emph{imitation from observation} methods such as Behavioral Cloning from Observation (BCO) and Generative Adversarial Imitation from Observation (GAIfO) learn directly from state-only demonstrations, typically by fitting an inverse dynamics model or by matching state-transition occupancy measures \citep{torabi2018bco,torabi2019gaifo}. Closer to our setting, ILPO learns a discrete latent action space from observations alone and then aligns the learned latent actions with real environment actions using a small amount of interaction \citep{edwards2019imitating}. LAPO, which learns latent action policies and world models from videos and shows that meaningful action structure can often be recovered from dynamics alone \citep{schmidt2024lapo}. Our work is complementary to these approaches: rather than proposing a new empirical objective for latent action discovery, we ask when latent actions and action-conditioned dynamics are identifiable in principle from $(o_t,o_{t+1},e)$.

Several recent works aim to leverage the abundance of action-free trajectories while using only a small amount of labeled data. APV studies representation learning from action-free videos for downstream RL \citep{seo2022reinforcement}, while VPT uses a small amount of labeled gameplay to train an inverse dynamics model that labels large-scale internet videos for policy pretraining \citep{baker2022vpt}. Semi-supervised offline RL formalizes datasets that mix action-labeled and action-free trajectories and studies practical pipelines that infer missing actions before applying standard offline RL algorithms \citep{zheng2023ssorl}. Our setting is more restrictive in one sense and more ambitious in another: we do not assume access to an external action labeler for the main identification argument, and instead exploit demonstrator identity as a source of structure that can render the latent action model identifiable.

Our use of demonstrator identity is related to recent work showing that heterogeneous datasets can make otherwise impossible action-free learning problems tractable. Most relevant is CRAFT, which studies action-free offline learning from multiple agents with differing policies in Ex-BMDPs and shows that differences across datasets can enable recovery of controllable latent structure \citep{levine2025craft}. While that work focuses on representation learning of controllable state features, our goal is different: we aim to identify a discrete latent action alphabet together with the corresponding action-conditioned transition kernels. In this sense, we use demonstrator diversity not only to learn useful representations, but to identify the semantics of the hidden decisions that generated the data.

Our perspective is also closely related to ideas from causality. At a high level, demonstrator identity acts as a source of exogenous policy variation: it changes the distribution over actions while, under our exclusion assumption, affecting the next observation only through the chosen action. This is reminiscent of instrumental-variable reasoning and of invariant prediction across environments \citep{hartford2017deepiv,arjovsky2019irm}. Recent work has begun to import such ideas into offline RL, for example to identify confounded transition dynamics when a valid instrument is available \citep{chen2022instrumental, liao2024instrumental}. Our setting differs in that the treatment itself---the action---is latent, and identification proceeds through a structured mixture decomposition rather than a standard regression problem.

Technically, our results build on the literature on identifiable nonnegative matrix factorization. In particular, minimum-volume criteria together with sufficiently scattered conditions can identify stochastic factorizations up to permutation \citep{fu2018identifiability}. Related factorization ideas have also been used to establish identifiability in latent-variable models such as hidden Markov models \citep{huang2018hmm}. We adapt these tools to the RL setting by interpreting the observable next-observation distributions as mixtures over latent action-conditioned transitions, indexed by demonstrator identity. This connection yields a clean identifiability theory for latent actions and dynamics from offline, action-free trajectories.

\section{Problem Setting}

We study a Markovian environment with \emph{latent} discrete actions. Let $O_t \in \mathcal O$ denote the observation at time $t$, where $\mathcal O$ is a measurable space, and let
$
A_t \in \mathcal A \coloneqq [k] = \{1,\dots,k\}
$
denote an unobserved action, where the number of latent actions $k$ is assumed known. Each trajectory is generated by one of $m$ demonstrators, indexed by a random variable $E \in [m]$.
Given the current observation $O_t=o$ and demonstrator identity $E=e$, the demonstrator chooses a latent action according to a policy
$
A_t \sim \pi_e^*(\cdot \mid o),
$
and the environment evolves according to an action-conditioned transition kernel
$
O_{t+1} \sim p^*(\cdot \mid o, A_t).
$
Actions are never observed. The offline data therefore consist only of tuples
$
(o_t, o_{t+1}, e),
$
drawn from trajectories generated by multiple demonstrators. We visualize this setting in \Cref{fig:scm}.

\begin{figure}
    \centering
    \begin{tikzpicture}[
          node distance=2cm and 2.2cm,
          ov/.style = {circle,draw,minimum size=3em},
          uv/.style = {rectangle,draw,dashed,minimum size=2.7em},
          edge/.style = {->, thick, >=Stealth},
          ]
      \node[ov] (Ot) at (0,0) {$O_t$};
      \node[ov] (Ot1) at (4,0) {$O_{t+1}$};
      \node[uv] (At) at (2,2) {$A_t$};
      \node[ov] (E) at (0,2) {$E$};

      \draw[edge] (Ot) -- (Ot1);
      \draw[edge] (Ot) -- (At);
      \draw[edge] (E) -- (At);
      \draw[edge] (At) -- (Ot1);
    \end{tikzpicture}
    \caption{Graphical model. Demonstrator identity $E$ affects the next observation $O_{t+1}$ only through the latent action $A_t$, while the current observation $O_t$ affects both action choice and the next observation.}
    \label{fig:scm}
\end{figure}
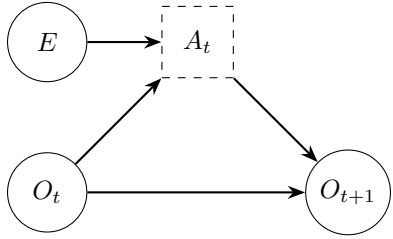

\paragraph{Observable and latent objects.}
At the population level, the observable object is the family of conditional distributions
$
\{p^*(\cdot \mid o,e) : o \in \mathcal O,\ e \in [m]\}.
$
The latent objects are the action-conditioned transition kernels and the demonstrator-specific policies
$
\{p^*(\cdot \mid o,a) : o \in \mathcal O,\ a \in [k]\}
$
and
$
\{\pi_e^*(a\mid o) : e \in [m],\ a \in [k],\ o \in \mathcal O\}.
$
By marginalizing over the latent action, for every $o \in \mathcal O$ and $e \in [m]$,
\begin{equation}
\label{eq:mixture}
p^*(\cdot \mid o,e)
=
\sum_{a=1}^k p^*(\cdot \mid o,a)\,\pi_e^*(a\mid o).
\end{equation}
Thus, for each fixed observation $o$, the demonstrator-conditioned next-observation law is a mixture of latent action-conditioned transition kernels, with mixing weights given by the demonstrator-specific policy.

\paragraph{Goal.}
Our goal is to recover the latent transition kernel and the demonstrator-specific policies
\[
p^*(\cdot \mid o,a), \qquad
\pi_e^*(a\mid o)
\]
from the observable family $\{p^*(\cdot \mid o,e)\}$. Since the labels of the latent actions are arbitrary, recovery can at best hold up to a permutation of the action indices.
We now state the assumptions used throughout.
\begin{assumption}[Structural assumptions]
\label{ass:basic}
Assume the following.
\begin{enumerate}[label=(\roman*)]
    \item \textbf{Sufficiency / Markov property:}
    $
    O_{t+1} \indep (O_{t-1},A_{t-1},O_{t-2},A_{t-2},\dots) \mid (O_t, A_t).
    $

    \item \textbf{Exclusion restriction:}
    $
    E \indep O_{t+1} \mid (O_t, A_t),
    $
    that is, demonstrator identity affects the next observation only through the latent action.

    \item \textbf{Well-defined latent transitions:}
    For every $o \in \mathcal O$ and every $a \in [k]$ such that
    $
    \pi_e^*(a\mid o) > 0
    \quad \text{for some } e \in [m],
    $
    the conditional distribution $p^*(\cdot \mid o,a)$ is well-defined.
\end{enumerate}
\end{assumption}

\section{Theoretical Guarantees}

Our analysis proceeds in four steps. First, we show that without demonstrator diversity the problem is non-identifiable in general. Second, we establish \emph{statewise} identifiability in finite observation spaces using identifiable stochastic (non-negative matrix factorization) NMF. Third, we give a extension to continuous observation spaces. Fourth, we show that continuity and a no-collision condition upgrade statewise permutations to a \emph{single global permutation}, which can then be fixed by a small amount of labeled action data.

We begin by formalizing the fact that observation-only data from a single behavior policy are not sufficient to identify latent actions and dynamics.

\begin{proposition}[Non-identifiability with a single demonstrator]
\label{prop:nonid-single}
Fix $o \in \mathcal O$. Suppose only the marginal next-observation law $p^*(\cdot \mid o)$ is observed, without demonstrator identity. Then in general the decomposition
\[
p^*(\cdot \mid o) = \sum_{a=1}^k p^*(\cdot \mid o,a)\,\pi^*(a\mid o)
\]
is not identifiable: there exist distinct pairs of latent transitions and action probabilities, not related by permutation, that induce the same observable distribution.
\end{proposition}

Proposition~\ref{prop:nonid-single} shows that additional structure is necessary. In our setting, demonstrator identity provides this structure by inducing observable variation in action probabilities while keeping the transition kernel fixed.

\subsection{Statewise identifiability in finite observation spaces}

We first consider the case where $\mathcal O$ is finite. Fix $o \in \mathcal O$. For each demonstrator $e \in [m]$, define
$
p_e^o \coloneqq p^*(\cdot \mid o,e) \in \mathbb R^{|\mathcal O|},
$ and $
t_a^o \coloneqq p^*(\cdot \mid o,a) \in \mathbb R^{|\mathcal O|},
$
and collect these into the matrices
$
P_o^*
\coloneqq
\begin{bmatrix}
p_1^o & \cdots & p_m^o
\end{bmatrix}
\in \mathbb R^{|\mathcal O| \times m},
$, $
T_o^*
\coloneqq
\begin{bmatrix}
t_1^o & \cdots & t_k^o
\end{bmatrix}
\in \mathbb R^{|\mathcal O| \times k},
$
and
$
\Pi_o^*
\coloneqq
\begin{bmatrix}
\pi_1^*(\cdot \mid o) & \cdots & \pi_m^*(\cdot \mid o)
\end{bmatrix}
\in \mathbb R^{k \times m}.
$
By \eqref{eq:mixture},
\begin{equation}
\label{eq:finite-factorization}
P_o^* = T_o^* \Pi_o^*.
\end{equation}
Each column of $P_o^*$, $T_o^*$, and $\Pi_o^*$ is a probability vector. Define the set of column-stochastic matrices
\[
\mathcal M_{r\times c}
\coloneqq
\left\{
M \in [0,1]^{r\times c}
\;\middle|\;
\mathbbm{1}^\top M = \mathbbm{1}^\top
\right\}.
\]
Then
$
P_o^* \in \mathcal M_{|\mathcal O|\times m},
$
$
T_o^* \in \mathcal M_{|\mathcal O|\times k},
$
and
$
\Pi_o^* \in \mathcal M_{k\times m}.
$
The finite-state identification problem is therefore: from the observable matrix $P_o^*$, recover the latent factors $(T_o^*, \Pi_o^*)$ up to permutation.
The factorization \eqref{eq:finite-factorization} is not unique in general: if $S \in \mathrm{GL}(k)$ is invertible, then
\[
P_o^* = T_o^* \Pi_o^* = (T_o^* S)(S^{-1}\Pi_o^*).
\]
Thus additional structure is needed to identify the latent factors. We use the sufficiently scattered condition of \citet{fu2018identifiability}, which ensures that among all feasible stochastic factorizations, the true one is the unique minimum-volume solution up to permutation.

\begin{definition}[Sufficiently scattered]
    Let $M \in \mathbb{R}^{u \times v}$ be nonnegative. Define
    $
    \mathrm{cone}(M) \coloneqq \{M\theta : \theta \ge 0\}
    $
    and 
    $
    C^* \coloneqq \{y : y^\top x \ge 0 \text{ for all } x \in C\}
    $
    for any cone $C$. For $v \ge 2$, define the second-order cones
    $
    \mathcal{C}_v
    \coloneqq
    \left\{
    x \in \mathbb{R}^v : \mathbbm{1}^\top x \ge \sqrt{v-1}\,\|x\|_2
    \right\}
    $
    and
    $
    \mathcal{C}_v^*
    \coloneqq
    \left\{
    x \in \mathbb{R}^v : \mathbbm{1}^\top x \ge \|x\|_2
    \right\}.
    $
    A matrix $H \in \mathcal{M}_{u\times v}$ is called \emph{sufficiently scattered} if
    \begin{equation*}
        \mathcal{C}_v \subseteq \mathrm{cone}(H^\top)
        \qquad\text{and}\qquad
        \mathrm{cone}(H^\top)^* \cap \mathrm{bd}(\mathcal{C}_v^*)
        =
        \{\lambda e_j : j \in [v], \ \lambda \ge 0\}.
    \end{equation*}
\end{definition}

\begin{assumption}[Finite-state identifiability conditions]
\label{ass:finite-id}
For all $o \in \mathcal O$ we have $|\mathcal O| \ge k$ and $m \ge k$ and 
\begin{enumerate}[label=(\roman*)]
    \item $\rank(P_o^*) = k$;
    \item $\Pi_o^*$ is sufficiently scattered.
\end{enumerate}
\end{assumption}
Intuitively, the sufficiently scattered condition means that the demonstrator policy vectors are not all small perturbations of one another or concentrated near a low-dimensional subset of the simplex; rather, they point in sufficiently many different action-mixture directions that their conic hull spreads broadly through the positive orthant, making the latent actions geometrically separable. In this sense it formalizes \emph{policy diversity}: different demonstrators must place meaningfully different relative masses on the latent actions, so that each latent action leaves a distinct signature in the observed mixtures instead of being confounded with the others.
We can now state the finite-state identifiability result.

\begin{theorem}[Statewise identifiability in finite observation spaces]
\label{thm:finite-ident}
    Assume \Cref{ass:finite-id}.
    Then for every $o \in \mathcal{O}$, the optimization problem
    \begin{equation}
    \label{eq:finite-minvol}
    \begin{aligned}
    \min_{T_o \in \mathcal{M}_{|\mathcal{O}|\times k},\, \Pi_o \in \mathcal{M}_{k\times m}}
    \quad & \det(T_o^\top T_o) \\
    \text{s.t.}\quad & P_o^* = T_o \Pi_o
    \end{aligned}
    \end{equation}
    recovers the latent factors up to a permutation of the action labels. More precisely, there exists a permutation matrix $\Sigma_o \in \{0,1\}^{k\times k}$ such that
    \[
    T_o = T_o^* \Sigma_o,
    \qquad
    \Pi_o = \Sigma_o^\top \Pi_o^*.
    \]
\end{theorem}
The minimum-determinant objective admits a simple geometric interpretation. Since
\[
\det(T_o^\top T_o)
\]
is the squared $k$-dimensional volume spanned by the columns of $T_o$, minimizing it selects, among all stochastic factorizations of $P_o^*$, the latent transition kernels that are least spread out while still explaining the demonstrator-conditioned transition laws. Intuitively, this is a conservative principle: we attribute to the latent actions only the variation that is forced by the observed diversity across demonstrators, and do not introduce more extreme latent action effects than necessary. If the demonstrators are sufficiently diverse, this still forces recovery of all action effects that are genuinely distinguishable from the data. On the other hand, if two latent actions induce the same transition law at state $o$, then they are observationally equivalent there and cannot be identified separately; in that case the true factorization becomes rank-deficient ($\rank(P_o^*) < k$) and $\det((T_o^*)^\top T_o^*)=0$. Nevertheless, if one reduces the problem to the effective dimension
$
r_o \coloneqq \operatorname{rank}(P_o^*),
$
which under sufficient policy diversity coincides with the number of distinct transition effects at $o$, then the same minimum-volume argument identifies the remaining distinct latent transitions (and the corresponding aggregated policy masses) up to permutation.

Theorem~\ref{thm:finite-ident} is a \emph{statewise} result: the permutation $\Sigma_o$ may still depend on $o$. The next subsections address both continuous observation spaces and alignment of these statewise permutations across $o$.

\subsection{Continuous-space identifiability via embedded Gram-determinant}

We now give a continuous-state analogue of the finite-state minimum-volume result that applies to \emph{general transition measures}, including deterministic transitions represented by Dirac masses.

Let $(\mathcal O,\mathcal B)$ be a measurable space, let $\mathcal P(\mathcal O)$ denote the set of probability measures on $(\mathcal O,\mathcal B)$, and let $\mathcal M(\mathcal O)$ denote the vector space of finite signed measures on $(\mathcal O,\mathcal B)$.
Fix a real Hilbert space $\mathcal H$ and a linear injective map
$
\Phi : \mathcal M(\mathcal O) \to \mathcal H.
$
A canonical example is a kernel mean embedding into an RKHS associated with a bounded characteristic kernel. Injectivity ensures that equality in embedding space implies equality of the underlying measures.
Fix $o \in \mathcal O$. Define the latent transition measures
$
t_a^o \coloneqq p^*(\cdot \mid o,a) \in \mathcal P(\mathcal O),
a \in [k],
$
and the observable demonstrator-conditioned transition measures
$
p^o_e \coloneqq p^*(\cdot \mid o,e) \in \mathcal P(\mathcal O),
e \in [m].
$
By the mixture identity,
$
p^o_e
=
\sum_{a=1}^k \pi_e^*(a)\,t_a^o,
e \in [m].
$
Applying the linear map $\Phi$ yields
\[
\Phi(p^o_e)
=
\sum_{a=1}^k \pi_e^*(a)\,\Phi(t_a^o).
\]
For any $k$-tuple of measures $(t_1,\dots,t_k)$, define the embedded Gram matrix
\[
G_\Phi(t_1,\dots,t_k)
\coloneqq
\big[
\langle \Phi(t_i), \Phi(t_j)\rangle_{\mathcal H}
\big]_{i,j=1}^k.
\]
Its determinant measures the squared volume spanned by the embedded latent transition measures in $\mathcal H$.
We also define the rank of a family of embedded observables by
$
\rank(P^*)
\coloneqq
\dim \operatorname{span}\{\Phi(p^o_1),\dots,\Phi(p^o_m)\}
$. Note that $
\Pi_o^*
=
\begin{bmatrix}
\pi_1^*(\cdot \mid o) & \cdots & \pi_m^*(\cdot \mid o)
\end{bmatrix}
\in \mathbb R^{k \times m}
$ remains a real-valued matrix and therefore the sufficiently scattered condition does not change.

\begin{assumption}[Continuous-state identifiability conditions]
\label{ass:cont-id}
For all $o \in \mathcal O$ we have $|\mathcal O| \ge k$ and $m \ge k$ and 
\begin{enumerate}[label=(\roman*)]
    \item $\rank\left(G_\Phi(t^o_1,\dots,t^o_k)\right) = k$;
    \item $\Pi_o^*$ is sufficiently scattered.
\end{enumerate}
\end{assumption}
We now state the continuous-state identifiability result.

\begin{theorem}[State-wise identifiability in continuous observation spaces]
\label{thm:continuous-ident-direct-correct}
Fix $o \in \mathcal O$ and suppress the dependence on $o$ in the notation.
Assume \Cref{ass:cont-id} and that the optimization problem
    \begin{equation}
    \label{eq:continuous-minvol-measure}
    \begin{aligned}
    \min_{\bar{t}_1,\dots,\bar{t}_k,\Pi}
    \quad & \det\big(G_\Phi(\bar{t}_1,\dots,\bar{t}_k)\big) \\
    \text{s.t.}\quad
    & p^o_e = \sum_{a=1}^k \pi_e(a)\bar{t}_a,
    \qquad e \in [m], \\
    & \Pi_o \in \mathcal M_{k\times m}, \\
    & \bar{t}_a \in \mathcal P(\mathcal O),
    \qquad a \in [k],
    \end{aligned}
    \end{equation}
admits an optimizer.
Then every optimizer of \eqref{eq:continuous-minvol-measure} is equal to the ground-truth factorization up to permutation: there exists a permutation matrix $\Sigma_o \in \{0,1\}^{k\times k}$ such that
$
\bar{t}_a = t^o_{\Sigma(a)},
a \in [k],
$
and
\[
\Pi_o = \Sigma_o^\top \Pi_o^*.
\]
\end{theorem}

In finite-dimensional stochastic NMF, optimizer existence is often automatic by compactness. In contrast, \eqref{eq:continuous-minvol-measure} is an optimization problem over an infinite-dimensional space of measures, so minimizing sequences need not converge without additional compactness assumptions.
A standard sufficient route is to work on a compact Polish observation space and choose $\Phi$ to be weakly continuous (for example, a kernel mean embedding associated with a bounded continuous kernel). Then $\mathcal P(\mathcal O)$ is weakly compact, $\mathcal M_{k\times m}$ is compact, the feasibility constraints are closed, and the objective
$
(t_1,\dots,t_k) \mapsto \det\big(G_\Phi(t_1,\dots,t_k)\big)
$
is continuous, so an optimizer exists by the direct method.

\subsubsection{From state-wise to global identifiability}

The state-wise identifiability results identify the latent action model separately at each observation $o$, but only up to an observation-dependent permutation. We now show that, under continuity and connectedness, these local permutations must agree globally.

Throughout this subsection, let $\Phi : \mathcal M(\mathcal O) \to \mathcal H$ be the fixed injective linear embedding used in the continuous-space theorem, where $\mathcal H$ is a real Hilbert space.
For each $o \in \mathcal O$ and $a \in [k]$, define the embedded true latent transitions
$
\phi_a^*(o) \coloneqq \Phi\!\big(p^*(\cdot \mid o,a)\big) \in \mathcal H,
$
and, for any candidate latent model, define the embedded recovered latent transitions
$
\phi_a(o) \coloneqq \Phi\!\big(p(\cdot \mid o,a)\big) \in \mathcal H.
$
We collect these into the operator-valued maps
\[
T^*(o) \coloneqq \begin{bmatrix} \phi_1^*(o) & \cdots & \phi_k^*(o) \end{bmatrix},
\qquad
T(o) \coloneqq \begin{bmatrix} \phi_1(o) & \cdots & \phi_k(o) \end{bmatrix}.
\]

\begin{assumption}
\label{ass:global_indetn}
    Assume that:
    \begin{enumerate}[label=(\roman*)]
        \item for each $a \in [k]$, the maps $o \mapsto \phi_a(o)$ and $o \mapsto \phi_a^*(o)$ are continuous from $\mathcal O$ into $\mathcal H$;
        \item for every $o \in \mathcal O$ and every $i \neq j$,
        $
        \phi_i^*(o) \neq \phi_j^*(o).
        $
    \end{enumerate}
\end{assumption}

\begin{theorem}[From local to global permutations]
\label{thm:global-perm}
Assume that $\mathcal O$ is a connected metric space. Assume \Cref{ass:global_indetn}. Assume that for each $o \in \mathcal O$ there exists a permutation matrix $\Sigma(o)$ such that
\[
T(o) = T^*(o)\Sigma(o),
\qquad
\Pi(o) = \Sigma(o)^\top \Pi^*(o).
\]
Then $\Sigma(o)$ is constant on $\mathcal O$. That is, there exists a single permutation matrix $\Sigma$ such that
\[
T(o) = T^*(o)\Sigma,
\qquad
\Pi(o) = \Sigma^\top \Pi^*(o)
\qquad
\text{for all } o \in \mathcal O.
\]
\end{theorem}

Theorem~\ref{thm:global-perm} shows that once the latent model is identified at each observation up to permutation, continuity and connectedness reduce the ambiguity to a single \emph{global} permutation of the latent action labels.
We now combine the state-wise identifiability theorem with Theorem~\ref{thm:global-perm} to obtain global identifiability.

\begin{assumption}[Global identifiability conditions]
\label{ass:main}
Assume the following.
\begin{enumerate}[label=(\roman*)]
    \item \label{ass:main:conn} $\mathcal O$ is a connected metric space.
    \item \label{ass:main:ss} For every $o \in \mathcal O$, the policy matrix $\Pi^*(o) \in \mathcal M_{k\times m}$ is sufficiently scattered.
    \item \label{ass:main:embed} For every $o \in \mathcal O$, $\rank\left(G_\Phi(t^o_1,\dots,t^o_k)\right) = k$.
    \item \label{ass:main:cont} For each $a \in [k]$, the map
    $
    o \mapsto \Phi\!\big(p^*(\cdot \mid o,a)\big)
    $
    is continuous from $\mathcal O$ into $\mathcal H$.
    \item \label{ass:main:anchor} There exists a nonempty subset $S \subseteq \mathcal O$ on which the global permutation is known.
\end{enumerate}
\end{assumption}

We give some intuition about the \Cref{ass:main}: Connectedness rules out choosing different action labelings on disconnected regions of the observation space and continuity prevents the local permutation ambiguity from changing abruptly across nearby observations (see \Cref{thm:global-perm}).
The sufficient-scattering condition guarantees state-wise identifiability by making the demonstrator policies diverse enough to separate the latent action-conditioned transition kernels and the positive-definiteness of the embedded Gram matrix ensures that the latent action effects are linearly independent, and hence distinguishable, at each observation (see \Cref{thm:continuous-ident-direct-correct}).

\begin{corollary}[Global identifiability]
\label{cor:main}
Assume \Cref{ass:main}(i)-(iv) hold. Let
$
o \mapsto \big(p(\cdot \mid o,1),\dots,p(\cdot \mid o,k),\Pi(o)\big)
$
be any continuous family such that for every $o \in \mathcal O$:
\begin{enumerate}[label=(\alph*)]
    \item the tuple
    $
    \big(p(\cdot \mid o,1),\dots,p(\cdot \mid o,k),\Pi(o)\big)
    $
    is feasible for \eqref{eq:continuous-minvol-measure}, and
    \item it attains the minimum in \eqref{eq:continuous-minvol-measure}.
\end{enumerate}
Then there exists a single permutation matrix $\Sigma$ such that
\[
p(\cdot \mid o,a) = p^*(\cdot \mid o,\Sigma(a)),
\qquad
\Pi(o) = \Sigma^\top \Pi^*(o)
\qquad
\text{for all } o \in \mathcal O.
\]
If, in addition, \Cref{ass:main}(v) holds, then $\Sigma = I_k$, and therefore
\[
p(\cdot \mid o,a) = p^*(\cdot \mid o,a),
\qquad
\Pi(o) = \Pi^*(o)
\qquad
\text{for all } o \in \mathcal O,\ a \in [k].
\]
In particular, the latent action-conditioned transition kernels and the demonstrator-specific policies are globally identifiable.
\end{corollary}

Although the assumptions in \Cref{ass:main} are natural, one may object that in applications they need not hold \emph{uniformly} over all $o \in \mathcal O$. In particular, some assumptions may fail on a small subset of the observation space or only at isolated observations. The key point is that the identifiability result is often robust to such localized violations: even when global identifiability fails on all of $\mathcal O$, one can typically still recover identifiability on the subset where the assumptions remain valid.
More concretely, the consequences of violating the assumptions can be summarized as follows.
\begin{itemize}
    \item If \Cref{ass:main}(i) fails, then $\mathcal O$ has multiple connected components. In that case, the argument of \Cref{thm:global-perm} applies separately on each connected component, so the latent actions remain identifiable up to a \emph{component-wise} permutation. If \Cref{ass:main}(v) holds on each component, then these permutations are fixed and one still obtains full identifiability on all of $\mathcal O$.

    \item If \Cref{ass:main}(ii) fails at some observations $o \in A \subsetneq \mathcal O$, then the geometric condition needed for the state-wise minimum-volume argument may break down on $A$. At such observations, state-wise identifiability is no longer guaranteed. Nevertheless, if $\mathcal O \setminus A$ remains connected and \Cref{ass:main}(ii) holds there, then the global permutation is still uniquely determined on $\mathcal O \setminus A$.

    \item If \Cref{ass:main}(iii) fails at some $o \in \mathcal O$, then the embedded latent transitions become linearly dependent at that observation. In particular, the action-conditioned effects are no longer fully distinguishable there, so one cannot expect unique recovery of all latent actions at that state. This is an intrinsic non-identifiability phenomenon rather than a limitation of our results.

    \item If \Cref{ass:main}(iv) fails, then the local permutation need no longer vary continuously with $o$. As a result, different permutations may be selected in different regions of the observation space. However, if the recovered and true embedded transitions are continuous on each element of a partition of $\mathcal O$, then the conclusion of \Cref{thm:global-perm} still applies piecewise, yielding one permutation per region.

    \item If \Cref{ass:main}(v) fails, then the latent actions remain identifiable only up to a single global permutation. This is the unavoidable label-swapping ambiguity familiar from latent-variable models.
\end{itemize}

\section{Estimation}

The identifiability results above are stated at the population level and characterize the latent model as the minimum-volume factorization of the observable conditional laws subject to exact mixture constraints and sufficient policy diversity. In practice, however, we only observe finitely many triples
$
(o_t,o_{t+1},e_t),
$
and in continuous observation spaces it is not tractable to optimize directly over arbitrary transition measures and state-dependent policy matrices. We therefore consider a parametric estimation procedure that is motivated by the identification theory but replaces the exact constrained problem by a relaxed empirical objective.

We parameterize the latent action-conditioned transition model and the demonstrator-specific latent policy by neural networks
\[
p_\theta(\cdot \mid o,a),
\qquad
\pi_\psi(a\mid o,e),
\]
where $\theta$ and $\psi$ denote trainable parameters. The policy network outputs a probability vector over the $k$ latent actions, for example via a softmax layer, and the transition model outputs a conditional distribution over next observations. In continuous spaces, $p_\theta(\cdot\mid o,a)$ may be instantiated as a Gaussian head, a mixture density model, or a conditional latent-state model.
The observable next-observation law induced by the model is
\begin{equation}
\label{eq:est-mixture-model}
p_{\theta,\psi}(o' \mid o,e)
=
\sum_{a=1}^k
\pi_\psi(a\mid o,e)\, p_\theta(o' \mid o,a),
\end{equation}
which mirrors the population mixture representation in \eqref{eq:mixture}.
A natural estimation principle is maximum likelihood under the observable mixture model \eqref{eq:est-mixture-model}. Given a dataset
$
\mathcal D = \{(o_i,o_i',e_i)\}_{i=1}^n,
$
we define the empirical negative log-likelihood
\begin{equation*}
\label{eq:est-nll}
\mathcal L_{\mathrm{fit}}(\theta,\psi)
=
-\frac1n
\sum_{i=1}^n
\log
\Big(
\sum_{a=1}^k
\pi_\psi(a\mid o_i,e_i)\,
p_\theta(o_i' \mid o_i,a)
\Big).
\end{equation*}
Minimizing $\mathcal L_{\mathrm{fit}}$ alone is generally insufficient: the model can explain the data with collapsed latent actions or nearly identical demonstrator policies, even when the true latent structure is identifiable.

\paragraph{A minimum-volume regularizer for latent transitions.}
The continuous-space identifiability result suggests selecting, among all factorizations that explain the data well, the one with minimum embedded volume. Let $\Phi : \mathcal M(\mathcal O) \to \mathcal H$ be the injective linear embedding used in \Cref{thm:continuous-ident-direct-correct}. For a candidate transition model, define the embedded latent transitions
\[
\mu_{\theta,a}(o)
\coloneqq
\Phi\big(p_\theta(\cdot\mid o,a)\big)
\in \mathcal H,
\qquad a\in[k].
\]
Their embedded Gram matrix is
\[
G_\theta(o)
\coloneqq
\big[
\langle \mu_{\theta,a}(o),\mu_{\theta,b}(o)\rangle_{\mathcal H}
\big]_{a,b=1}^k.
\]
A practical surrogate for the population minimum-volume principle is then
\begin{equation*}
\label{eq:est-rvol}
\mathcal R_{\mathrm{vol}}(\theta)
=
\frac1n
\sum_{i=1}^n
\log\det\!\big(G_\theta(o_i)+\varepsilon I_k\big),
\end{equation*}
where $\varepsilon>0$ is a small numerical regularization constant. Minimizing $\mathcal R_{\mathrm{vol}}$ encourages the latent transition family to be as simple and compact as possible while still explaining the observed data.

\paragraph{Preventing policy collapse.}
In the identifiability theory, demonstrator diversity is not an objective to be maximized, but a condition on the feasible factorizations: the true policy matrix must be sufficiently scattered. A direct practical analogue is therefore not to reward arbitrarily large diversity, but to penalize \emph{degenerate} policy geometries in which the learned policy matrix becomes nearly low-rank or some latent actions become unused.
For a fixed observation $o$, define the learned policy matrix
\[
\Pi_\psi(o)
\coloneqq
\begin{bmatrix}
\pi_\psi(\cdot\mid o,1) & \cdots & \pi_\psi(\cdot\mid o,m)
\end{bmatrix}
\in \mathbb R^{k\times m}.
\]
To discourage collapse, we introduce a diversity barrier based on the log-determinant of the policy Gram matrix:
\begin{equation*}
\label{eq:est-rpol}
\mathcal R_{\mathrm{pol}}(\psi)
=
\frac1n
\sum_{i=1}^n
\Big[
\tau
-
\log\det\!\big(
\Pi_\psi(o_i)\Pi_\psi(o_i)^\top + \varepsilon I_k
\big)
\Big]_+,
\end{equation*}
where $[x]_+ = \max\{x,0\}$ and $\tau$ is a user-specified threshold. This term vanishes whenever the policy matrix is sufficiently well-conditioned, and becomes active only when demonstrator policies begin to collapse toward a low-dimensional or poorly separated configuration. In this way, it acts as a soft feasibility constraint rather than a competing objective.

\paragraph{Optional label anchoring.}
As in the theoretical development, the latent action labels are only identifiable up to permutation unless a small amount of side information is available. If a small labeled dataset
$
\mathcal D_{\mathrm{lab}}
=
\{(o_j,e_j,a_j^\star)\}_{j=1}^r
$
is available, then the residual global permutation ambiguity can be fixed by adding the supervision term
\begin{equation*}
\label{eq:est-anchor}
\mathcal L_{\mathrm{anchor}}(\psi)
=
-\frac1r
\sum_{j=1}^r
\log \pi_\psi(a_j^\star \mid o_j,e_j).
\end{equation*}

\paragraph{Final objective.}
Combining the terms above yields the estimator
\begin{equation}
\label{eq:est-objective}
\min_{\theta,\psi}
\quad
\mathcal L_{\mathrm{fit}}(\theta,\psi)
+
\lambda_{\mathrm{vol}}\,\mathcal R_{\mathrm{vol}}(\theta)
+
\lambda_{\mathrm{pol}}\,\mathcal R_{\mathrm{pol}}(\psi)
+
\lambda_{\mathrm{anchor}}\,\mathcal L_{\mathrm{anchor}}(\psi),
\end{equation}
where $\lambda_{\mathrm{vol}},\lambda_{\mathrm{pol}},\lambda_{\mathrm{anchor}}\ge 0$ are tuning parameters. The roles of the terms are complementary: $\mathcal L_{\mathrm{fit}}$ enforces agreement with the observed demonstrator-conditioned transitions, $\mathcal R_{\mathrm{vol}}$ selects a simple latent action representation, $\mathcal R_{\mathrm{pol}}$ keeps the learned policies in a sufficiently diverse regime, and $\mathcal L_{\mathrm{anchor}}$ resolves the final global label symmetry when labeled actions are available.

The objective \eqref{eq:est-objective} should be viewed as a practical surrogate for the population identification problem rather than as an exact finite-sample analogue of the theory. In particular, the minimum-volume principle in \Cref{thm:continuous-ident-direct-correct} is applied there over \emph{exactly feasible} factorizations, whereas \eqref{eq:est-objective} trades off data fit and regularization in finite samples. The policy regularizer is therefore important in practice: without it, the model may reduce the transition volume by collapsing distinct latent actions or by learning nearly identical demonstrator policies, even when the ground-truth factorization is identifiable. The estimator above is designed to preserve the geometric intuition of the theory: learn the simplest latent action-conditioned transition model that explains the data, while requiring the demonstrator policies to remain sufficiently diverse to reveal that latent structure.

\paragraph{Computing the embedded Gram matrix.}
The form of $G_\theta(o)$ depends on how the transition model $p_\theta(\cdot\mid o,a)$ is represented. If $p_\theta$ is an explicit density model on $\mathcal O$, such as a Gaussian or mixture-density head, and the embedding $\Phi$ is induced by a kernel $k$, then the Gram entries can be written as
\[
G_\theta(o)_{ab}
=
\langle \Phi(p_\theta(\cdot\mid o,a)),\Phi(p_\theta(\cdot\mid o,b))\rangle_{\mathcal H}
=
\mathbb E_{x\sim p_\theta(\cdot\mid o,a),\,y\sim p_\theta(\cdot\mid o,b)}[k(x,y)],
\]
which may be available in closed form for simple choices of $p_\theta$ and $k$, or else approximated by Monte Carlo samples. If the transition model is deterministic, say
$
p_\theta(\cdot\mid o,a)=\delta_{f_\theta(o,a)},
$
then
$
G_\theta(o)_{ab}
=
k(f_\theta(o,a),f_\theta(o,b)),
$
so the Gram matrix reduces to pairwise kernel evaluations between the predicted next observations. Finally, if the transition model is defined in a learned latent state space, for example by encoding observations as $z=h_\omega(o)$ and modeling $p_\theta(z'\mid z,a)$, then the same construction may be applied in latent space by replacing $\mathcal O$ with the latent space and computing the Gram matrix from the corresponding latent transition distributions. In all cases, the role of $G_\theta(o)$ is the same: it quantifies how distinct the action-conditioned transitions are at observation $o$, and its determinant serves as a measure of the effective volume of the learned latent action simplex.

\section{Conclusion}

We studied when latent actions and action-conditioned dynamics can be identified from offline trajectories in which actions are never observed but demonstrator identity is available. Our main insight is that demonstrator diversity can make the latent structure recoverable: when demonstrators induce sufficiently rich variation in latent action usage and the transition dynamics are shared, the observable next-observation laws admit an identifiable factorization up to permutation of the latent actions. We established this first in finite observation spaces via identifiable stochastic nonnegative matrix factorization, and then extended the result to continuous observation spaces through an embedded minimum-volume criterion. Under continuity and connectedness assumptions, the remaining statewise permutation ambiguity reduces to a single global permutation, which can be fixed with a small amount of labeled action data.
Beyond the identifiability results, our analysis suggests a practical principle for learning from action-free offline data: learn the simplest latent action-conditioned transition model consistent with the data while requiring demonstrator policies to remain sufficiently diverse to reveal that latent structure. We hope these results provide a useful theoretical foundation for future algorithms that leverage heterogeneous passive data for representation learning, model learning, and offline reinforcement learning without action annotations.

\bibliographystyle{abbrvnat}
\bibliography{refs}

\appendix

\newpage

\section{Proofs}

\subsection[g]{Proof of \Cref{prop:nonid-single}}
It suffices to give one example. Let $\mathcal O = \{1,2\}$ and $k=2$, and consider the observable distribution
\[
p^*(\cdot \mid o) = \begin{bmatrix} 1/2 \\ 1/2 \end{bmatrix}.
\]
One feasible decomposition is
\[
t_1 = \begin{bmatrix} 1 \\ 0 \end{bmatrix},
\qquad
t_2 = \begin{bmatrix} 0 \\ 1 \end{bmatrix},
\qquad
\pi^*(\cdot \mid o) = \begin{bmatrix} 1/2 \\ 1/2 \end{bmatrix},
\]
for which
\[
p^*(\cdot \mid o) = \tfrac12 t_1 + \tfrac12 t_2.
\]
A different feasible decomposition is
\[
\tilde t_1 = \begin{bmatrix} 3/4 \\ 1/4 \end{bmatrix},
\qquad
\tilde t_2 = \begin{bmatrix} 1/4 \\ 3/4 \end{bmatrix},
\qquad
\tilde\pi^*(\cdot \mid o) = \begin{bmatrix} 1/2 \\ 1/2 \end{bmatrix},
\]
for which
\[
p^*(\cdot \mid o) = \tfrac12 \tilde t_1 + \tfrac12 \tilde t_2.
\]
These two decompositions are not related by permutation. Hence the latent transitions and latent actions are not identifiable from a single observation-only conditional distribution in general.

\subsection[f]{Proof of \Cref{thm:finite-ident}}

\begin{theorem}[Theorem 1 of \citet{fu2018identifiability}]
\label{thm:fu}
Let $u,v,z \in \mathbb{N}$ with $u,z \ge v$. Suppose
\[
X^* \in \mathbb{R}^{u\times z},
\qquad
W^* \in \mathbb{R}^{u\times v},
\qquad
H^* \in \mathbb{R}^{v\times z}
\]
satisfy
\[
X^* = W^* H^*,
\qquad
\rank(X^*) = \rank(H^*) = v,
\]
and assume that $H^*$ is sufficiently scattered. Then the optimization problem
\begin{equation}
\label{eq:fu-optimization}
\begin{aligned}
    \min_{W \in \mathbb{R}^{u\times v}, H \in \mathbb{R}^{v\times z}}
    \quad & \det(W^\top W) \\
    \text{s.t.}\quad & X^* = WH\\
    \quad & H \mathbbm{1}^{\top} = \mathbbm{1}, H \geq 0
\end{aligned}
\end{equation}
has a unique solution up to permutation: there exists a permutation matrix $\Sigma \in \{0,1\}^{v\times v}$ and a full-rank permutation matrix $D \in \mathbb{R}^{v\times v}$ such that
\[
W = W^* \Sigma D,
\qquad
H = D^{-1} \Sigma^\top H^*.
\]
\end{theorem}

\begin{lemma}[Column-stochasticity removes diagonal scaling]
\label{lem:no-scaling}
Let $T,T' \in \mathcal M_{u\times k}$ and $\Pi,\Pi' \in \mathcal M_{k\times m}$ satisfy
\[
T' = T D,
\qquad
\Pi' = D^{-1}\Pi
\]
for some invertible diagonal matrix $D \in \mathbb R^{k\times k}$. Then $D = I_k$.
\end{lemma}

\begin{proof}
Write $D=\mathrm{diag}(d_1,\dots,d_k)$. Since $T$ and $T'$ are column-stochastic,
\[
\mathbbm{1}^\top T' = \mathbbm{1}^\top T D = \mathbbm{1}^\top D = \mathbbm{1}^\top.
\]
Hence $d_a = 1$ for every $a \in [k]$, so $D=I_k$.
\end{proof}

Under Assumption~\ref{ass:finite-id}, \Cref{thm:fu} implies that any feasible factorization of $P_o^*$ is unique up to permutation and diagonal scaling. Concretely, there exist a permutation matrix $\Sigma_o$ and an invertible diagonal matrix $D_o$ such that
    \[
    T_o = T_o^* \Sigma_o D_o,
    \qquad
    \Pi_o = D_o^{-1}\Sigma_o^\top \Pi_o^*.
    \]
    By Lemma~\ref{lem:no-scaling}, column-stochasticity forces $D_o = I_k$. Therefore
    \[
    T_o = T_o^* \Sigma_o,
    \qquad
    \Pi_o = \Sigma_o^\top \Pi_o^*,
    \]
    as claimed.

\subsection[d]{Proof of \Cref{thm:continuous-ident-direct-correct}}

\begin{lemma}[Correct determinant bound]
\label{lem:det-bound-cont-correct}
Let $\Pi^* \in \mathbb R_+^{k\times m}$ have rank $k$, and define
\[
K \coloneqq \operatorname{cone}(\Pi^*) \subset \mathbb R^k.
\]
Assume
\[
\mathcal C_k \subseteq K,
\qquad
K^* \cap \operatorname{bd}(\mathcal C_k^*)
=
\{\lambda e_j : j\in[k],\ \lambda \ge 0\}.
\]
Let $A \in \mathrm{GL}(k)$ satisfy
\[
A\Pi^* \ge 0,
\qquad
A^\top \mathbbm 1 = \mathbbm 1.
\]
Then
\[
|\det(A)| \le 1.
\]
Moreover, equality holds if and only if $A$ is a permutation matrix.
\end{lemma}

\begin{proof}
Let $r_1^\top,\dots,r_k^\top$ denote the rows of $A$.

Since $A\Pi^* \ge 0$, for each $i$ we have
\[
r_i^\top \Pi^* \ge 0.
\]
Equivalently, for every column $\pi_j^*$ of $\Pi^*$,
\[
r_i^\top \pi_j^* \ge 0.
\]
Hence each row vector $r_i$ lies in the dual cone $K^*$.

Because $\mathcal C_k \subseteq K$, duality of cones gives
\[
K^* \subseteq \mathcal C_k^*
=
\{x \in \mathbb R^k : \mathbbm 1^\top x \ge \|x\|_2\}.
\]
Therefore each row satisfies
\[
\|r_i\|_2 \le \mathbbm 1^\top r_i.
\]

By Hadamard's inequality applied to the rows of $A$,
\[
|\det(A)| \le \prod_{i=1}^k \|r_i\|_2
\le
\prod_{i=1}^k \mathbbm 1^\top r_i.
\]
Now
\[
\mathbbm 1^\top r_i = (A\mathbbm 1)_i,
\]
so
\[
\prod_{i=1}^k \mathbbm 1^\top r_i
=
\prod_{i=1}^k (A\mathbbm 1)_i.
\]
Also,
\[
\sum_{i=1}^k (A\mathbbm 1)_i
=
\mathbbm 1^\top A \mathbbm 1
=
(A^\top \mathbbm 1)^\top \mathbbm 1
=
\mathbbm 1^\top \mathbbm 1
=
k.
\]
Since each $(A\mathbbm 1)_i \ge \|r_i\|_2 \ge 0$, AM--GM yields
\[
\prod_{i=1}^k (A\mathbbm 1)_i \le 1.
\]
Hence
\[
|\det(A)| \le 1.
\] 
Now suppose equality holds. Then equality must hold in:
\begin{enumerate}[label=(\roman*)]
    \item Hadamard's inequality, so the rows $r_1,\dots,r_k$ are pairwise orthogonal;
    \item the bound $\|r_i\|_2 \le \mathbbm 1^\top r_i$, so each $r_i \in \operatorname{bd}(\mathcal C_k^*)$;
    \item AM--GM, so each $(A\mathbbm 1)_i = 1$, i.e. each row sum equals $1$.
\end{enumerate}
Thus each row $r_i$ lies in
\[
K^* \cap \operatorname{bd}(\mathcal C_k^*),
\]
so by assumption,
\[
r_i = \lambda_i e_{\ell_i}
\]
for some $\lambda_i \ge 0$ and $\ell_i \in [k]$.
Since the row sum is $1$, we get $\lambda_i = 1$.
Thus every row is a standard basis vector.
Because $A$ is invertible, these basis vectors must be distinct, so $A$ is a permutation matrix.

Conversely, any permutation matrix satisfies the assumptions and has determinant of absolute value $1$.
\end{proof}

\begin{proof}[Proof of \Cref{thm:continuous-ident-direct-correct}]
Fix $o \in \mathcal O$ and suppress the dependence on $o$ in the notation. Write
\[
t_a \coloneqq t_a^o,
\qquad
p_e \coloneqq p_e^o,
\qquad
\Pi^* \coloneqq \Pi_o^*.
\]
Let $(\bar t_1,\dots,\bar t_k,\Pi)$ be any feasible solution of
\eqref{eq:continuous-minvol-measure}.

Define the embedded latent and observable elements
\[
u_a \coloneqq \Phi(t_a) \in \mathcal H,
\qquad
\bar u_a \coloneqq \Phi(\bar t_a) \in \mathcal H,
\qquad
v_e \coloneqq \Phi(p_e) \in \mathcal H,
\]
and collect them into linear maps
\[
T^* = \begin{bmatrix} u_1 & \cdots & u_k \end{bmatrix},
\qquad
\bar T = \begin{bmatrix} \bar u_1 & \cdots & \bar u_k \end{bmatrix},
\qquad
P = \begin{bmatrix} v_1 & \cdots & v_m \end{bmatrix}.
\]
By the ground-truth mixture identity, feasibility, and linearity of $\Phi$,
\[
P = \bar T \Pi = T^* \Pi^*.
\]

By \Cref{ass:cont-id}(i),
\[
\rank\!\left(G_\Phi(t_1,\dots,t_k)\right)=k,
\]
so the Gram matrix $G_\Phi(t_1,\dots,t_k)$ is positive definite. Hence
$u_1,\dots,u_k$ are linearly independent, and therefore
\[
\dim \operatorname{span}\{u_1,\dots,u_k\} = k.
\]
Moreover, since $\Pi^*$ is sufficiently scattered, the cone
\[
K \coloneqq \operatorname{cone}(\Pi^*)
\]
contains $\mathcal C_k$, which is full-dimensional in $\mathbb R^k$. Hence $K$
is full-dimensional, so $\rank(\Pi^*)=k$.

Therefore the columns of $\Pi^*$ span $\mathbb R^k$, and thus
\[
\operatorname{span}\{v_1,\dots,v_m\}
=
T^*(\mathbb R^k)
=
\operatorname{span}\{u_1,\dots,u_k\}.
\]
In particular,
\[
\dim \operatorname{span}\{v_1,\dots,v_m\}=k.
\]
On the other hand, since $P=\bar T\Pi$, we have
\[
\operatorname{span}\{v_1,\dots,v_m\}
\subseteq
\operatorname{span}\{\bar u_1,\dots,\bar u_k\}.
\]
The left-hand side has dimension $k$, while the right-hand side is spanned by at
most $k$ vectors, so in fact
\[
\operatorname{span}\{\bar u_1,\dots,\bar u_k\}
=
\operatorname{span}\{u_1,\dots,u_k\}.
\]
Therefore there exists a unique invertible matrix $S \in \mathrm{GL}(k)$ such that
\[
\bar T = T^* S.
\]
Substituting into $P=\bar T\Pi=T^*\Pi^*$ gives
\[
T^* S \Pi = T^* \Pi^*.
\]
Since the columns of $T^*$ are linearly independent, $T^*$ is injective on
$\mathbb R^k$, and hence
\[
S\Pi = \Pi^*,
\qquad\text{so}\qquad
\Pi = S^{-1}\Pi^*.
\]
Let
\[
A \coloneqq S^{-1}.
\]
Then
\[
A\Pi^* = \Pi \ge 0.
\]
Next, because $\Phi$ is linear and injective, the identity $\bar T=T^*S$ implies
that for each $a \in [k]$,
\[
\bar t_a = \sum_{b=1}^k S_{ba}\, t_b
\]
as finite signed measures in $\mathcal M(\mathcal O)$. Since both $\bar t_a$ and
each $t_b$ are probability measures, taking total mass yields
\[
1
=
\bar t_a(\mathcal O)
=
\sum_{b=1}^k S_{ba}\, t_b(\mathcal O)
=
\sum_{b=1}^k S_{ba}.
\]
Thus every column of $S$ sums to one, i.e.
\[
\mathbbm 1^\top S = \mathbbm 1^\top.
\]
Equivalently,
\[
A^\top \mathbbm 1 = \mathbbm 1.
\]
We may therefore apply \Cref{lem:det-bound-cont-correct} to conclude that
\[
|\det(A)| \le 1,
\qquad\text{hence}\qquad
|\det(S)| \ge 1,
\]
with equality if and only if $S$ is a permutation matrix.

Finally, the embedded Gram matrix transforms as
\[
G_\Phi(\bar t_1,\dots,\bar t_k)
=
S^\top G_\Phi(t_1,\dots,t_k) S,
\]
and therefore
\[
\det\big(G_\Phi(\bar t_1,\dots,\bar t_k)\big)
=
\det\big(G_\Phi(t_1,\dots,t_k)\big)\,\det(S)^2.
\]
Since $|\det(S)| \ge 1$, we obtain
\[
\det\big(G_\Phi(\bar t_1,\dots,\bar t_k)\big)
\ge
\det\big(G_\Phi(t_1,\dots,t_k)\big).
\]
The ground-truth factorization $(t_1,\dots,t_k,\Pi^*)$ is feasible, so it achieves
the value
\[
\det\big(G_\Phi(t_1,\dots,t_k)\big).
\]
Hence every optimizer must satisfy $|\det(S)|=1$. By
\Cref{lem:det-bound-cont-correct}, this happens if and only if $S$ is a
permutation matrix. Therefore there exists a permutation matrix $\Sigma$ such that
\[
\bar T = T^* \Sigma,
\qquad
\Pi = \Sigma^\top \Pi^*.
\]
Because $\Phi$ is injective, $\bar T = T^* \Sigma$ implies
\[
\bar t_a = t_{\Sigma(a)},
\qquad a \in [k].
\]
Restoring the dependence on $o$ gives exactly the claimed conclusion:
\[
\bar t_a = t^o_{\Sigma(a)},
\qquad a \in [k],
\qquad\text{and}\qquad
\Pi_o = \Sigma_o^\top \Pi_o^*.
\]
\end{proof}

\subsection[f]{Proof of \Cref{thm:global-perm}}
Fix $o_0 \in \mathcal O$, and let $\Sigma_0 \coloneqq \Sigma(o_0)$. Write $\sigma_0$ for the corresponding permutation of $[k]$, so that
\[
\phi_a(o_0) = \phi_{\sigma_0(a)}^*(o_0)
\qquad
\text{for all } a \in [k].
\]
Because the true embedded latent transitions at $o_0$ are pairwise distinct, the quantity
\[
\delta_0
\coloneqq
\min_{i\neq j}
\big\|
\phi_i^*(o_0)-\phi_j^*(o_0)
\big\|_{\mathcal H}
\]
is strictly positive.

By continuity of $o \mapsto \phi_a(o)$ and $o \mapsto \phi_a^*(o)$, there exists a neighborhood $U$ of $o_0$ such that for all $o \in U$ and all $a \in [k]$,
\[
\|\phi_a(o)-\phi_a(o_0)\|_{\mathcal H} < \delta_0/4,
\qquad
\|\phi_a^*(o)-\phi_a^*(o_0)\|_{\mathcal H} < \delta_0/4.
\]
Fix $o \in U$. For each $a \in [k]$, we then have
\[
\|\phi_a(o)-\phi_{\sigma_0(a)}^*(o)\|_{\mathcal H}
\le
\|\phi_a(o)-\phi_a(o_0)\|_{\mathcal H}
+
\|\phi_{\sigma_0(a)}^*(o_0)-\phi_{\sigma_0(a)}^*(o)\|_{\mathcal H}
< \delta_0/2.
\]
On the other hand, for any $b \neq \sigma_0(a)$,
\begin{align*}
\|\phi_a(o)-\phi_b^*(o)\|_{\mathcal H}
&\ge
\|\phi_{\sigma_0(a)}^*(o_0)-\phi_b^*(o_0)\|_{\mathcal H}
-
\|\phi_a(o)-\phi_a(o_0)\|_{\mathcal H}
-
\|\phi_b^*(o)-\phi_b^*(o_0)\|_{\mathcal H} \\
&>
\delta_0 - \delta_0/4 - \delta_0/4
=
\delta_0/2.
\end{align*}
Thus, for every $a \in [k]$, the vector $\phi_a(o)$ is strictly closer than $\delta_0/2$ to $\phi_{\sigma_0(a)}^*(o)$ and strictly farther than $\delta_0/2$ from every other true component $\phi_b^*(o)$. Therefore the only permutation relating $T(o)$ and $T^*(o)$ is $\Sigma_0$. Hence
\[
\Sigma(o)=\Sigma_0
\qquad
\text{for all } o \in U.
\]
So $\Sigma(\cdot)$ is locally constant.

Since $\Sigma(o)$ takes values in the finite discrete set of $k\times k$ permutation matrices, any locally constant map is constant on each connected component of $\mathcal O$. Because $\mathcal O$ is connected, $\Sigma(o)$ must be constant on all of $\mathcal O$.

\subsection[p]{Proof of \Cref{cor:main}}
Fix any $o \in \mathcal O$. By \Cref{ass:main}~\ref{ass:main:ss}--\ref{ass:main:embed}, Theorem~\ref{thm:continuous-ident-direct-correct} applies at that observation. Hence there exists a permutation matrix $\Sigma(o)$ such that
\[
p(\cdot \mid o,a) = p^*(\cdot \mid o,\Sigma(o)(a)),
\qquad
\Pi(o) = \Sigma(o)^\top \Pi^*(o).
\]
Because $\Phi$ is injective, this is equivalent to
\[
T(o) = T^*(o)\Sigma(o),
\qquad
\Pi(o) = \Sigma(o)^\top \Pi^*(o),
\]
where
\[
T(o)=\begin{bmatrix}
\Phi(p(\cdot \mid o,1)) & \cdots & \Phi(p(\cdot \mid o,k))
\end{bmatrix},
\qquad
T^*(o)=\begin{bmatrix}
\Phi(p^*(\cdot \mid o,1)) & \cdots & \Phi(p^*(\cdot \mid o,k))
\end{bmatrix}.
\]
By \Cref{ass:main}~\ref{ass:main:cont} and the continuity assumed for the candidate family, the maps $o \mapsto T(o)$ and $o \mapsto T^*(o)$ are continuous. Moreover, \Cref{ass:main}~\ref{ass:main:embed} implies that the embedded true latent transitions are linearly independent at every $o$, and in particular pairwise distinct. Therefore the hypotheses of \Cref{thm:global-perm} are satisfied, so there exists a single permutation matrix $\Sigma$ such that
\[
p(\cdot \mid o,a) = p^*(\cdot \mid o,\Sigma(a)),
\qquad
\Pi(o) = \Sigma^\top \Pi^*(o)
\qquad
\text{for all } o \in \mathcal O.
\]

Finally, by \Cref{ass:main}~\ref{ass:main:anchor}, the permutation is known on the anchor subset $S$. Since the same permutation $\Sigma$ applies globally, it is uniquely determined there and hence uniquely determined everywhere. If the anchor identifies the true labeling, then $\Sigma = I_k$, yielding
\[
p(\cdot \mid o,a) = p^*(\cdot \mid o,a),
\qquad
\Pi(o) = \Pi^*(o)
\qquad
\text{for all } o \in \mathcal O,\ a \in [k].
\]
This proves the claim.

\end{document}